\documentclass[a4paper,11pt]{article}

\usepackage{amssymb}
\usepackage{graphicx}
\usepackage{verbatim}
\usepackage{xcolor}
\usepackage[caption=false]{subfig}
\usepackage{listings}
\usepackage{hyperref}
\usepackage{tikz}
\usepackage{pgfplots} 
\usepackage{float}
\usepackage{appendix}
\usepackage{xspace}
\usepackage{booktabs}
\usepackage{makecell}
\usepackage{microtype}
\usepackage{authblk}
\usepackage{amsmath}
\usepackage{amsthm}

\pgfplotsset{compat=1.18} 

\newcommand{\eofex}{\hfill$\blacksquare$}

\newcommand*{\PL}{\mathrm{PL}}

\newcommand*{\size}{\mathrm{size}}
\newcommand*{\common}{likeness}
\newcommand*{\sat}{\mathtt{sat}}

\newcommand{\FPol}{\ensuremath{\protect\mathsf{FP}}\xspace}
\newcommand{\FPNPlogwit}{\ensuremath{\FPol^{\NP}[log, wit]}\xspace}
\newcommand{\FPSigmaPlogwit}[1]{\ensuremath{\FPol^{\Sigma^p_{#1}}[log, wit]}\xspace}
\newcommand{\NP}{\mathsf{NP}}
\newcommand{\coNP}{\mathsf{coNP}}

\newcommand\subsetsim{\mathrel{%
  \ooalign{\raise0.2ex\hbox{$\subset$}\cr\hidewidth\raise-0.8ex\hbox{\scalebox{0.9}{$\sim$}}\hidewidth\cr}}}

\author[1]{Tobias Geibinger}
\author[2]{Reijo Jaakkola}
\author[2]{Antti Kuusisto}
\author[1]{Xinghan Liu}
\author[2]{Miikka Vilander}
\affil[1]{TU Wien, Austria}
\affil[2]{Tampere University, Finland}

\date{}

\begin{document}

\setlength\abovedisplayskip{3pt}
\setlength\belowdisplayskip{3pt}
\title{Why this and not that? A Logic-based Framework for Contrastive Explanations}

\theoremstyle{plain}
\newtheorem{theorem}{Theorem}[section]
\newtheorem{lemma}[theorem]{Lemma}
\newtheorem{corollary}[theorem]{Corollary}
\newtheorem{proposition}[theorem]{Proposition}
\newtheorem{fact}[theorem]{Fact}
\theoremstyle{definition}
\newtheorem{definition}[theorem]{Definition}
\newtheorem{remark}[theorem]{Remark}
\newtheorem{example}[theorem]{Example}

\maketitle

\begin{abstract}
\noindent
   We define several canonical problems related to contrastive explanations, each answering a question of the form ``Why P but not Q?''. The problems compute causes for both P and Q, explicitly comparing their differences. We investigate the basic properties of our definitions in the setting of propositional logic. We show, inter alia, that our framework captures a cardinality-minimal version of existing contrastive explanations in the literature. Furthermore, we provide an extensive analysis of the computational complexities of the problems. We also implement the problems for CNF-formulas using answer set programming and present several examples demonstrating how they work in practice.
\end{abstract}


\section{Introduction}

The importance of explanations for decisions made by automatic classifiers has been well-established with the rise of AI methods. In this work, we investigate a category of explanations called \emph{contrastive explanations}, which answer the question ``Why $P$, but not $Q$?'' These types of questions are very common in practical contexts, when an expected outcome was not obtained. It has also been argued \cite{MILLER20191} that even when not explicitly asking for one, people often prefer an explanation in the form of a comparison between the situation as it occurred in reality and a different one that could have happened.

In this work, we use a logic-based framework to formalize several problems where the task is to find contrastive explanations. Related problems have been studied previously, for example, by Darwiche \cite{darwiche23} and Ignatiev et al. \cite{ignatiev20}, where, broadly speaking, the goal is to find a \emph{counterfactual explanation}: ``Had \(H\) been true, it would have been the case that \(Q\)''. In contrast to these works, solutions to our problems explicitly answer both ``Why \(P\)?'' and ``Why not \(Q\)?'' with dedicated output formulas that are required to be structurally similar. Our definitions have been partially inspired by the work of Lipton~\cite{lipton1990}, where he argues that a contrastive explanation should contrast the causes of \(P\) against the absence of corresponding causes of \(Q\). 

Our first problem, which we call the contrastive explanation problem, aims to explain why two seemingly similar entities have differing properties. For example, we might ask why two individuals with similar backgrounds were assigned different credit risk levels. An input to our problem consists of two sets of formulas \(S,S'\) along with two formulas \(\varphi,\psi\) such that \(S \models \varphi \land \neg \psi\) and \(S' \models \neg \varphi \land \psi\). The output of the problem is a minimal size triple \((\theta,\theta',\chi)\) of formulas, such that \(\theta \land \chi\) explains why $S$ implies $\varphi \land \neg \psi$, contrasting with \(\theta' \land \chi\) that explains why $S'$ implies $\neg \varphi \land \psi$. Formally, we require that \(S \vDash \theta \land \chi \vDash \varphi \land \neg \psi\) and \(S' \vDash \theta' \land \chi \vDash \neg \varphi \land \psi\). See Figure \ref{fig:enter-label} for an illustration.
Our problem formulation naturally enforces the similarity of the two explanations, as it encourages shifting content from \(\theta\) and \(\theta'\) into the common formula \(\chi\). 
To see the intuition behind this, consider a situation where we have two animals and we ask why one of them is a cat while the other is a dog. Here one would not insert ``can not fly'' into the differentiating formulae \(\theta\) and \(\theta'\), since it holds for both cats and dogs. On the other hand, it might be necessary to include ``can not fly'' to \(\chi\), since it separates cats and dogs from, say, crows.

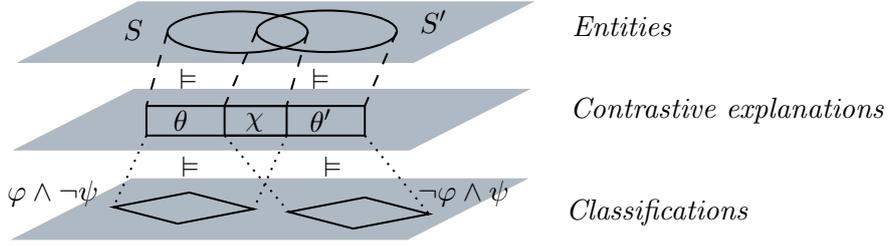
\begin{figure}
    \centering

\tikzset{every picture/.style={line width=0.75pt}} 

\begin{tikzpicture}[x=0.75pt,y=0.75pt,yscale=-1,xscale=1]

\draw  [color={rgb, 255:red, 255; green, 255; blue, 255 }  ,draw opacity=1 ][fill={rgb, 255:red, 156; green, 168; blue, 181 }  ,fill opacity=0.8 ] (239.53,40.87) -- (440,40.87) -- (377.47,72.87) -- (177,72.87) -- cycle ;
\draw  [color={rgb, 255:red, 255; green, 255; blue, 255 }  ,draw opacity=1 ][fill={rgb, 255:red, 156; green, 168; blue, 181 }  ,fill opacity=0.8 ] (237.53,84.87) -- (438,84.87) -- (375.47,116.87) -- (175,116.87) -- cycle ;
\draw  [color={rgb, 255:red, 255; green, 255; blue, 255 }  ,draw opacity=1 ][fill={rgb, 255:red, 156; green, 168; blue, 181 }  ,fill opacity=0.8 ] (236.53,129.87) -- (437,129.87) -- (374.47,161.87) -- (174,161.87) -- cycle ;
\draw   (254.5,56.87) .. controls (254.5,51.1) and (270.17,46.43) .. (289.5,46.43) .. controls (308.83,46.43) and (324.5,51.1) .. (324.5,56.87) .. controls (324.5,62.63) and (308.83,67.3) .. (289.5,67.3) .. controls (270.17,67.3) and (254.5,62.63) .. (254.5,56.87) -- cycle ;
\draw   (299,56.43) .. controls (299,50.67) and (314.67,46) .. (334,46) .. controls (353.33,46) and (369,50.67) .. (369,56.43) .. controls (369,62.2) and (353.33,66.87) .. (334,66.87) .. controls (314.67,66.87) and (299,62.2) .. (299,56.43) -- cycle ;
\draw  [dash pattern={on 4.5pt off 4.5pt}]  (324.5,56.87) -- (314,94) ;
\draw  [dash pattern={on 4.5pt off 4.5pt}]  (299,56.43) -- (283,94) ;
\draw  [dash pattern={on 4.5pt off 4.5pt}]  (369,56.43) -- (353,94) ;
\draw   (227.5,144.87) -- (265.4,137.53) -- (297.5,145.87) -- (259.6,153.21) -- cycle ;
\draw   (315.1,147.34) -- (353,140) -- (385.1,148.34) -- (347.21,155.68) -- cycle ;
\draw   (244,94) -- (314,94) -- (314,108.87) -- (244,108.87) -- cycle ;
\draw   (283,94) -- (353,94) -- (353,108.87) -- (283,108.87) -- cycle ;
\draw  [dash pattern={on 0.84pt off 2.51pt}]  (244,108.87) -- (227.5,144.87) ;
\draw  [dash pattern={on 0.84pt off 2.51pt}]  (283,108.87) -- (315.1,147.34) ;
\draw  [dash pattern={on 0.84pt off 2.51pt}]  (314,108.87) -- (297.5,145.87) ;
\draw  [dash pattern={on 0.84pt off 2.51pt}]  (353,108.87) -- (385.1,148.34) ;
\draw  [dash pattern={on 4.5pt off 4.5pt}]  (254.5,56.87) -- (244,94) ;

\draw (231,49.4) node [anchor=north west][inner sep=0.75pt]    {$S$};
\draw (379,44.4) node [anchor=north west][inner sep=0.75pt]    {$S'$};
\draw (173,130) node [anchor=north west][inner sep=0.75pt]    {$\varphi \land \neg \psi $};
\draw (378,130) node [anchor=north west][inner sep=0.75pt]    {$\neg \varphi \land \psi $};
\draw (256,95) node [anchor=north west][inner sep=0.75pt]    {$\theta $};
\draw (292,96) node [anchor=north west][inner sep=0.75pt]    {$\chi $};
\draw (324,94) node [anchor=north west][inner sep=0.75pt]    {$\theta '$};
\draw (455,47.4) node [anchor=north west][inner sep=0.75pt]    {\textit{Entities}};
\draw (454,88.4) node [anchor=north west][inner sep=0.75pt]    {\textit{Contrastive explanations}};
\draw (452,139.4) node [anchor=north west][inner sep=0.75pt]    {\textit{Classifications}};
\draw (259,74) node [anchor=north west][inner sep=0.75pt]    {$\vDash $};
\draw (325,74) node [anchor=north west][inner sep=0.75pt]    {$\vDash $};
\draw (260,118) node [anchor=north west][inner sep=0.75pt]    {$\vDash $};
\draw (331,118) node [anchor=north west][inner sep=0.75pt]    {$\vDash $};

\end{tikzpicture}
    
    \caption{Illustration of the contrastive explanation problem}
    \vspace{0.2cm}
    \label{fig:enter-label}
\end{figure}

Setting \(S = \{\varphi \land \neg \psi\}\) and \(S' = \{\neg \varphi \land \psi\}\) in the contrastive explanation problem, we obtain as a special case a problem that directly compares all the models of $\varphi \land \neg \psi$ with those of $\neg \varphi \land \psi$, giving a global contrast between the two formulas. Accordingly, we call this the global contrastive problem. We also define an alternative global problem which we call the minimal separator problem. In this problem the goal is to find a single minimal difference that suffices to separate the two input formulas.

We then move on to consider a variant of this problem where the input does not contain \(S'\) and the question we want to answer is ``why does \(S\) satisfy \(\varphi\) but not \(\psi\)?''. We formalize this problem in two distinct ways, both rooted in counterfactual reasoning. In the first problem, which we call the counterfactual contrastive explanation problem, the output is a minimal size triple \((\theta,\theta',\chi)\) such that \(S \models \theta \land \chi \models  \varphi \land \neg \psi\) and \(\theta' \land \chi \models \neg \varphi \land \psi\). The second problem, called the counterfactual difference problem, is otherwise the same except that we require that \(S \equiv \theta \land \chi\), i.e., that \(\theta \land \chi\) defines the current state of affairs \(S\). Roughly speaking, in the first problem we want to find a reason for \(S \models \varphi \land \neg \psi\) and a cause for \(\neg \varphi \land \psi\) which are similar, while in the second problem we want to find a minimal modification to \(S\) that would guarantee that \(\neg \varphi \land \psi\) holds.

We investigate our problems more closely in the setting of propositional logic ($\PL$), assuming for simplicity that the output formulas are in conjunctive normal form. We first show that our definitions indeed find contrasts between the input formulas. For example, for the global contrastive explanation problem, we show that each clause $C$ of the output formula $\theta$ is a \emph{weak contrast} between $\varphi$ and $\psi$, meaning $\varphi$ entails $C$ while $\psi$ does not. We also show a link between our definitions and previous work on contrastive explanations: if the set \(S\) in the input of the counterfactual difference problem defines an assignment, then the output corresponds to a cardinality-minimal CXp 
\cite{ignatiev20}. As a by-product, we show that both of our counterfactual problems output partial assignments up to equivalence when presented with partial assignment inputs. 


We also study the computational complexity of our problems in the case of \(\PL\). We first observe that the contrastive explanation problem, the global contrastive problem and the minimal separator problem are all \(\Sigma_2^p\)-complete. We then show that certain natural variants of our counterfactual problems are also \(\Sigma_2^p\)-complete. 
Finally, we provide a prototypical implementation of our problems using Answer Set Programming, which given the complexity of the underlying problems, is an adequate computation formalism \cite{EiterIK09}. We use this implementation to demonstrate how our problems work in practice via three case studies.




\paragraph{\textbf{Related Work}}

We now present related work on contrastive explanation within the broader context of logic-based explainable AI. It was originally for \emph{local explanation}, namely explaining why a classifier makes a certain decision for a \emph{given} input instance. 
As summarized in several works~\cite{ignatiev20,silva22,darwiche23}, local explanation answers one of the following two questions:
     Q1. (Why): what minimal aspects of an instance guarantee the actual decision?
     Q2. (Why not): what minimal changes to an instance result in a different decision?
An answer to Q1 is nowadays commonly called a \emph{sufficient reason}
\cite{DBLP:conf/ecai/DarwicheH20} or an AXp (short for abductive explanation) \cite{ignatiev2019abduction}. A sufficient reason is a minimal subset of its features' values s.t. changing any other feature values will not change the classification.
In $\PL$, a sufficient reason coincides with a \emph{prime implicant} of the Boolean classifier $\varphi$ which is \emph{locally true} in the given instance.
Therefore the concept was first introduced in the literature under the name PI-explanation \cite{shih2018formal}.
For Barcelo et al. \cite{barcelo2020model}, a sufficient reason itself need not be minimal, and it is called a \emph{minimal/minimum sufficient reason}, if it satisfies the subset-/cardinality-minimality requirement respectively.
We refer to these notions as instances of \emph{direct explanation}, for lack of a better term.

Dually, an answer to Q2 is commonly referred to as either a \emph{contrastive explanation} (CXp) \cite{ignatiev20}, a \emph{counterfactual explanation} \cite{stepin2021survey} or a necessary reason \cite{darwiche23}.
In all cases it is a subset-minimal part of the instance, changing the values of which results in a different classification. 
The duality of AXp and CXp is well-established from a logical viewpoint via minimal hitting sets \cite{ignatiev20}.
Namely, an AXp is a subset-minimal intersection of every CXp of the instance and vice versa.
Similarly to before, a shift from subset-minimality to cardinality-minimality brings us to the notion of \emph{minimum change required} by Barcelo et al. \cite{barcelo2020model}.

We now move from local to global explanation.
Note that an explanation can be called `global' in either \emph{conditional} or \emph{categorical} sense, namely either 1) it is a local explanation \emph{if an input instance satisfies it}, or 2) it explains the whole classifier.
In the former sense, the duality between AXp and CXp extends naturally to \emph{global AXp} and \emph{counterexamples} \cite{silva22}. In propositional logic, they coincide with prime implicants and negated prime implicates.
An example of a global direct explanation in the latter sense is the (shortest) prime DNF expression of Boolean classifiers \cite{explainability1}.
However, here one cannot obtain a corresponding contrastive explanation in terms of the duality, because the contrastive explanations referred to so far are by nature localized to the actual case, viz., they are counterfactual.
One must define a categorically global contrastive explanation in a genuinely global manner, which is a key contribution of our paper.

   

Besides the aforementioned ones, recently Bassan et al. \cite{bassanlocal} define \emph{global sufficient reason} as the set of features which is a sufficient reason for all instances.
Similarly, \emph{global contrastive reason} is ``a subset of features that may cause a misclassification for any possible input''. This notion of a global contrastive explanation differs significantly from ours, which aims to describe essentially all the differences between two classifiers.

\section{Preliminaries}

We will formulate our various explanation problems for an arbitrary logic. For the purposes of this work, a \textbf{logic} $\mathcal{L}$ is a tuple $L = (\mathcal{F}, \vDash, \sat)$, where $\mathcal{F}$ is a set of formulas, $\mathop{\vDash} \subseteq \mathcal{P}(\mathcal{F}) \times \mathcal{F}$ is a binary logical consequence relation and $\sat \subseteq \mathcal{F}$ is a unary satisfiability relation. Instead of $\varphi, \psi \in \mathcal{F}$ and $\{\varphi\} \vDash \psi$ we will write $\varphi, \psi \in \mathcal{L}$ and $\varphi \vDash \psi$ for simplicity. If for $S, S' \subseteq \mathcal{F}$ we have $S \vDash \psi$ for every $\psi \in S'$ and $S \vDash \varphi$ for every $\varphi \in S$, then we say that $S$ and $S'$ are equivalent, denoted $S \equiv S'$. We formulate our definitions for logics $\mathcal{L}$ with classical conjunction $\land$ and classical negation $\neg$ with the usual semantics. These connectives are not necessary but they help to keep our definitions clean and readable.


As the logics that we consider are two-valued, the above framework works best when modeling classifiers for binary classification. To model classifiers with more than two classes, we can assign to each class \(c\) a formula \(\varphi_c\) which is satisfied precisely by the inputs classified as \(c\). We call such a formula a \textbf{class formula}. This approach has been used, for example, in \cite{darwiche23}.


In addition to our very general definitions,
we will also consider the important special case of propositional logic. Let $\tau$ be a set of \textbf{proposition symbols} called a \textbf{vocabulary}. The set $\mathrm{PL}[\tau]$ of formulas of \textbf{propositional logic} over \(\tau\) is generated by the 
grammar
    $\varphi ::= \bot \mid p \mid \neg \varphi \mid \varphi \land \varphi \mid \varphi \lor \varphi$,
where $p \in \tau$. We define $\top := \neg \bot$ as an abbreviation. 

A function $s : \sigma \to \{0,1\}$, where \(\sigma \subseteq \tau\), is called a \textbf{partial} \(\tau\)-\textbf{assignment}. When \(\sigma = \tau\), we say that \(s\) is a \(\tau\)-\textbf{assignment}. We often identify (partial) assignments with conjunctions of literals in the obvious way. The semantics of $\mathrm{PL}[\tau]$ is defined as usual, i.e., for a \(\tau\)-assignment \(s\) we define $s \nvDash \bot$ always, $s \vDash p$ if $s(p) = 1$, $s \vDash \neg \psi$ if $s \nvDash \psi$, $s \vDash \psi \land \theta$ if $s \vDash \psi$ and $s \vDash \theta$, and finally $s \vDash \psi \lor \theta$ if $s \vDash \psi$ or $s \vDash \theta$. Given \(S \subseteq \mathrm{PL}[\tau]\) and a (partial) \(\tau\)-assignment \(s\), we say that \(S\) \textbf{defines} \(s\), if \(S \equiv s\). 
If $S = \{\varphi\}$, we write $\varphi$ defines $s$ rather than $\{\varphi\}$ defines $s$.

Assignments that satisfy a formula are also called its \textbf{models}. For two formulas $\varphi, \psi \in \mathrm{PL}[\tau]$, we use the notation $\varphi \vDash \psi$ for \textbf{logical consequence}, that is $\varphi \vDash \psi$ if for all assignments $s$, $s \vDash \varphi$ implies $s \vDash \psi$. 
We extend this notation also to sets $S, T \subseteq \mathrm{PL}[\tau]$ of formulas, defining $S \vDash T$ if $\bigwedge S \vDash \bigwedge T$.

Formulas of the form $p$ or $\neg p$, where $p$ is a proposition symbol, are called \textbf{literals}. The \textbf{dual} \(\overline{\ell}\) of a literal \(\ell\) is defined as $\overline{p} = \neg p$ and $\overline{\neg p} = p$. A disjunction of literals is called a \textbf{clause}. A formula $\varphi \in \mathrm{PL}[\tau]$ is in \textbf{conjunctive normal form}, or CNF, if $\varphi$ is a conjunction of clauses. We will often denote clauses as sets of literals and CNF-formulas as sets of clauses. For example, if a CNF-formula $\varphi$ has the clause $\neg p \lor q$, we might write $\{\neg p,q\} \in \varphi$. We denote the set of all CNF formulas $\varphi \in \mathrm{PL}[\tau]$ by $\mathrm{CNF}[\tau]$.

Formula size is a key notion in our work. We do not define formula size in the general case as many different definitions could make sense depending on the logic considered. We will, however, consider the size of propositional formulas in practice so we define it here. We define the \textbf{size} of a formula $\varphi \in \mathrm{PL}[\tau]$ simply to be number of occurences of proposition symbols. More formally, we define it recursively as follows: $\size(\bot) = 0$ and $\size(p) = 1$, for a proposition symbol $p$; \(\size(\neg \psi) = \size(\psi)\); $\size(\psi_1 \land \psi_2) = \size(\psi_1 \lor \psi_2) = \size(\psi_1) + \size(\psi_2)$.

\section{Generalized Contrastive Explanation}\label{sec:notions}

In this section we introduce several natural problems that concern contrastive explanations. We start by defining problems focused on explaining differences either between instances (objects) or between properties. After this we consider problems that deal with the problem of explaining why a given instance has a particular property. In our logic-based framework instances are represented as sets of formulas while properties correspond to single formulas.

\subsection{Comparison Explanations}

We begin with a problem, where we have two different instances $S$ and $S'$ at hand, and we want to know why one of them satisfies $\varphi$ while the other satisfies $\psi$. In practice, the two instances could seem very similar to us, leading to the question of key differences that explain their different properties $\varphi$ and $\psi$. Van Bouwel and Weber \cite{VanBouwel2002437} call this question an O-contrast, also cited by Miller \cite{MILLER20191}. 

\begin{definition}\label{def:twoinputproblem}
    The \textbf{contrastive explanation problem} is defined as follows.

    \noindent\textbf{Input:} A tuple $(S, S', \varphi, \psi)$, where $S, S' \subseteq \mathcal{L}$ are finite sets and $\varphi, \psi \in \mathcal{L}$.

    \noindent\textbf{Output:} A triple $(\theta, \theta', \chi)$, where $\theta, \theta', \chi \in \mathcal{L}$ have the following properties.
    \begin{enumerate}
        \item $S \vDash \theta \land \chi \vDash \varphi \land \neg \psi$ and $S' \vDash \theta' \land \chi \vDash \neg \varphi \land \psi$,
        \item $\size(\theta)+\size(\theta') + \size(\chi)$ is minimal,
        \item as a secondary optimization criterion, $\size(\chi)$ is maximal.
    \end{enumerate}%
    If no such triple exists, output $\mathtt{error}$.

\end{definition}%
The output contains the contrast formulas $\theta$ and $\theta'$ as well as a shared context formula $\chi$. \emph{The essential part of the output are the differentiating formulas \(\theta\) and \(\theta'\)}.
The three conditions can be motivated as follows. Condition 1 ensures that \(\theta \land \chi\) and \(\theta' \land \chi\) serve as explanations for why \(S\) and \(S'\) satisfy \(\varphi \land \neg \psi\) and \(\neg \varphi \land \psi\), respectively. Condition 2 requires these explanations to be minimal to avoid including irrelevant information. Condition 3 enforces similarity between the explanations by maximizing the common context formula \(\chi\), as it is desirable to shift content from \(\theta\) and \(\theta'\) into \(\chi\).

In the above definition we are asking ``Why $S \models \varphi \land \neg \psi$ and $S' \models \neg \varphi \land \psi$?'' as opposed to ``Why $S \models \varphi$ and $S' \models \psi$?''. The reason we do this is that, in accordance to Lipton~\cite{lipton1990}, we allow the formulas \(\varphi\) and \(\psi\) to be compatible, but we seek explanations which entail that only one of them is true. We will follow this approach throughout the paper.

\begin{example}\label{example:first-example}
We give a concrete example of Definition \ref{def:twoinputproblem} in the case where \(\mathcal{L}\) is propositional logic. Consider the propositional vocabulary \(\tau = \{p,q,r\}\). Let \(\varphi\) be a formula of \(\PL[\tau]\) which is satisfied by precisely those assignments which map exactly two propositional symbols to 1. Furthermore, let \(\psi\) be a formula of \(\PL[\tau]\) which is satisfied by precisely those assignments which map exactly one propositional symbol to 1. Note that \(\varphi \models \neg \psi\) and \(\psi \models \neg \varphi\), whence \(\varphi \land \neg \psi \equiv \varphi\) and \(\neg \varphi \land \psi \equiv \psi\). Now, let \(S := \{p,q,\neg r\}\) and \(S' := \{p,\neg q,\neg r\}\). The following triple is a possible solution: \(\theta := q , \ \theta' := \neg q \text{ and } \chi := p \land \neg r\). \eofex
\end{example}
Consider the special case of the contrastive explanation problem, where $S = \{\varphi \land \neg \psi\}$ and $S' = \{\neg \varphi \land \psi\}$. Now we are asking for a contrast between all of the models of $\varphi$ and $\psi$, with no limitation to a more specific locality $S$ or $S'$. Thus we are asking ``What is the difference between $\varphi$ and $\psi$?'' We call this case the global contrastive explanation problem and define it next separately.

\begin{definition} \label{def:globaltwoinputproblem}
    The \textbf{global contrastive explanation problem} is defined as follows.

    \noindent\textbf{Input:} A pair $(\varphi, \psi)$, where $\varphi, \psi \in \mathcal{L}$.

    \noindent\textbf{Output:} A triple $(\theta, \theta', \chi)$, where $\theta, \theta', \chi \in \mathcal{L}$ have the following properties.
    \begin{enumerate}
        \item $\theta\land \chi \equiv \varphi \land \neg \psi$ and $\theta'\land \chi \equiv \neg \varphi \land \psi$,
        \item $\size(\theta)+\size(\theta') + \size(\chi)$ is minimal,
        \item as a secondary optimization criterion, $\size(\chi)$ is maximal.
    \end{enumerate}
    %
\end{definition}
\begin{example}\label{example:global-contrastive}
Consider again the formulas \(\varphi\) and \(\psi\) from Example \ref{example:first-example}. Restricting our output formulas to CNF-formulas for readability, the following triple is a solution to the global contrastive explanation problem.
\begin{align*}
    \theta &:= (p \lor r) \land (q \lor r) \land (p \lor q) \\
    \theta' &:= (\neg p \lor \neg r) \land (\neg q \lor \neg r) \land (\neg p \lor \neg q) \\
    \chi &:= (p \lor q \lor r) \land (\neg p \lor \neg q \lor \neg r)
\end{align*}
These formulas tell us the following. First, \(\theta\) says that in every model of \(\varphi \land \neg \psi \equiv \varphi\) at least two propositional symbols are true. Secondly, \(\theta'\) says that in every model of \(\neg \varphi \land \psi \equiv \psi\) at least two propositional symbols are false. Finally, \(\chi\) tells us that in models of \(\varphi \lor \psi\)
one propositional symbol is true and one is false.\eofex
\end{example}
Another notion related to contrastivity is that of a separator. A separator of $\varphi$ from $\psi$ is a property that all models of $\varphi$ and no models of $\psi$ have. A separator can thus also be seen as a different answer to the question ``What is the difference between $\varphi$ and $\psi$?'' The next problem asks for a minimal separator of $\varphi$ from $\psi$.

\begin{definition}
    The \textbf{minimal separator problem} is defined as follows.

    \noindent\textbf{Input:} A pair $(\varphi, \psi)$, where $\varphi, \psi \in \mathcal{L}$.

    \noindent\textbf{Output:} A formula $\theta \in \mathcal{L}$ such that
    \(\varphi \vDash \theta, \ \psi \vDash \neg \theta \text{ and } \size(\theta) \text{ is minimal}\). If no such $\theta$ exists, output $\mathtt{error}$.
\end{definition}
%
The global contrastive explanation problem can be seen as giving ``all'' separators between $\varphi$ and $\psi$, or at least enough to achieve equivalent formulas, while the minimal separator problem only gives one minimal separator.  In terms of the natural language question ``What is the difference between $\varphi$ and $\psi$?'', giving all differences or a single difference could both be considered reasonable answers.

\begin{example}
Consider again the formulas \(\varphi,\psi\) in Example \ref{example:first-example}. Clearly \(\varphi \models \neg \psi\). The formula \((p \lor r) \land (q \lor r) \land (p \lor q)\) is a minimal separator between \(\varphi\) and \(\psi\).\eofex
\end{example}
\subsection{Counterfactual Explanations}

We have seen above how the contrastive explanation problem answers the question ``Why does $S$ satisfy $\varphi$ while $S'$ satisfies $\psi$?'' We now turn our attention to the case, where we have only one instance $S$ at hand and we are asking about the classification of that instance. The question here is ``Why does $S$ satisfy $\varphi$ and not $\psi$?''. Van Bouwel and Weber \cite{VanBouwel2002437} call this question a P-contrast, also cited by Miller \cite{MILLER20191}.
We define two problems that modify the contrastive explanation problem in different ways to answer this question.

Our first problem is a straightforward modification of Definition \ref{def:twoinputproblem}, where we simply leave $S'$ to be existentially quantified. 
The condition $S' \vDash \theta' \land \chi \vDash \neg \varphi \land \psi$ is then reduced to the form $\theta' \land \chi \vDash \neg \varphi \land \psi$ since $S'$ itself is not actually needed for the output. The definition is as follows.


\begin{definition}\label{def:problem:one_input_consequence}
    The \textbf{counterfactual contrastive explanation problem} is defined as follows.
    
    \noindent \noindent\textbf{Input:} A tuple $(S, \varphi, \psi)$, where $S \subseteq \mathcal{L}$ is a finite set and $\varphi, \psi \in \mathcal{L}$.

    \noindent \noindent\textbf{Output:} A triple $(\theta, \theta', \chi)$, where $\theta, \theta', \chi \in \mathcal{L}$ have the following properties.
    \begin{enumerate}
        \item $S \vDash \theta\land\chi \vDash \varphi \land \neg \psi$ and $\theta'\land\chi \vDash \neg \varphi \land \psi,$ 
        \item $\theta' \land \chi$ is satisfiable iff $S$ is satisfiable,
        \item $\size(\theta)+\size(\theta') + \size(\chi)$ is minimal,
        \item as a secondary optimization criterion, $\size(\chi)$ is maximal.
    \end{enumerate}
    If no such triple exists, output $\mathtt{error}$.
\end{definition}
The above definition places emphasis on minimal reasons for why instances satisfy properties. Another way to think about contrastive explanations is to consider minimal changes required to the input instance in order to achieve the desired outcome. This can be seen in the literature in the concept of CXp \cite{ignatiev20}. It turns out that minimal reasons for satisfaction and minimal changes to achieve satisfaction do not always coincide. This difference between reasons, or why-questions, and changes, or what-is-the-difference-questions motivates our next definition. The intuitive question here is ``What is the difference between $S$ that satisfies $\varphi$ and the closest $S'$ that instead satisfies $\psi$?'' 

\begin{definition}\label{def:problem:one_input_equivalence}
    The \textbf{counterfactual difference problem} is defined as follows.
    
    \noindent\textbf{Input:} A tuple $(S, \varphi, \psi)$, where $S \subseteq \mathcal{L}$ is a finite set and $\varphi, \psi \in \mathcal{L}$.

    \noindent\textbf{Output:} A triple $(\theta, \theta', \chi)$, where $\theta, \theta', \chi \in \mathcal{L}$ have the following properties.
    \begin{enumerate}
        \item $S \equiv \theta\land\chi \vDash \varphi \land \neg \psi$ and $\theta'\land\chi \vDash \neg \varphi \land \psi,$ 
        \item $\theta' \land \chi$ is satisfiable iff $S$ is satisfiable,
        \item $\size(\theta)+\size(\theta') + \size(\chi)$ is minimal,
        \item as a secondary optimization criterion, $\size(\chi)$ is maximal.
    \end{enumerate}
    If no such triple exists, output $\mathtt{error}$.
\end{definition}
\begin{remark}
    Note that for both Definitions \ref{def:problem:one_input_consequence} and \ref{def:problem:one_input_equivalence}, if $\psi \vDash \varphi$ and \(S\) is satisfiable, then the output is $\mathtt{error}$. Here the models of $\psi$ are included in those of $\varphi$ and thus $\neg \varphi \land \psi$ is a contradiction. The input can be repaired by replacing $\varphi$ with $\varphi \land \neg \psi$, since \(\neg (\varphi \land \neg \psi) \land \psi \equiv \psi\). This makes the two input formulas separate and preserves the original question: ``Why does $S$ satisfy $\varphi$ but not $\psi$?''
\end{remark}
%


%
\begin{example}\label{example:sea-bird-example}
    Assume we have two propositional classifiers, $\varphi$ and $\psi$, trained on some data about seabirds. The classifier $\varphi = \mathtt{beak\_pouch}$ classifies a bird as a pelican if the bird has a distinctive pouch on the underside of its beak. The classifier $\psi = \neg \mathtt{beak\_pouch} \land \mathtt{small} \land ((\mathtt{white\_body} \land \mathtt{webbed\_feet}) \lor \mathtt{grey\_wing})$ classifies a bird as a seagull if it has no beak pouch, is less than 1 meter in size and either has white plumage on its body and webbed feet, or a grey wing. Note that these are not complete descriptions of these types of birds but rather classification criteria that could have been extracted from a dataset. 

    Let 
    $S = \{\mathtt{beak\_pouch},\neg \mathtt{small},\mathtt{white\_body},\mathtt{webbed\_feet},\neg \mathtt{grey\_wing}\}$
    be a description of a bird. We want to know why this bird was classified as a pelican and not as a seagull. Starting with Definition \ref{def:problem:one_input_equivalence}, the output in this case is $\theta = \mathtt{beak\_pouch}$, $\theta' =\mathtt{small}$, $\chi = (\mathtt{\neg beak\_pouch} \lor \mathtt{\neg small}) \land \mathtt{white\_body} \land \mathtt{webbed\_feet} \land \neg \mathtt{grey\_wing}$. 
    We can read the solution as ``This bird has a beak pouch and is large, so it's a pelican. If it had no beak pouch and was small, it would instead be a seagull.''
    Note how all attributes of the bird are listed in the formula $\theta \land \chi$ even though \texttt{beak\_pouch} suffices to classify the bird as a pelican.

    The output of Definition \ref{def:problem:one_input_consequence} is $\theta = \mathtt{beak\_pouch}$, $\theta' = \neg \mathtt{beak\_pouch} \land \mathtt{small} \land \mathtt{grey\_wing}$, $\chi = \top$. This would be read as ``This bird has a beak pouch so it's a pelican. If it had no beak pouch, was small and had a grey wing, it would instead be a seagull.'' This time only the beak pouch is listed in $\theta \land \chi$ as it suffices to classify the bird as a pelican.
    
    For another difference between the definitions, note that out of the two options provided by the classifier $\psi$, Definition \ref{def:problem:one_input_consequence} has chosen the one with the grey wing. Another option in the search space would have been $\theta = \mathtt{beak\_pouch}$, $\theta' = \neg \mathtt{beak\_pouch} \land \mathtt{small}$, $\chi = \mathtt{white\_body} \land \mathtt{webbed\_feet}$. Out of these two, the grey wing option was chosen because the reasons $\theta \land \chi$ and $\theta' \land \chi$ given for why the bird was a pelican or a seagull were shorter in the first option. This illustrates the fact that Definition \ref{def:problem:one_input_consequence} is not concerned with minimal differences between the input and the counterfactual case, but rather minimal differences between minimal reasons for the classifications of the input and the counterfactual case.\eofex
\end{example}
\section{The Case of Propositional Logic}

In this section, we investigate our problems more closely in the setting of propositional logic. We show that our problems indeed find contrasts between the inputs in a semantic sense. We also establish a formal link between our definitions and existing work on contrastive explanations. We conclude with a study of the computational complexity of our problems.


\subsection{Contrasts and Likenesses}\label{subsec:proplogicproperties}

We start by defining some notions that compare two formulas $\varphi$ and $\psi$. 

\begin{definition}
    Let $\varphi, \psi, \theta$ be formulas of a logic \(\mathcal{L}\).
    \begin{enumerate}
        \item $\theta$ is a \textbf{weak $(\varphi, \psi)$-contrast} if $\varphi \land \neg \psi \vDash \theta$ and $\neg \varphi \land \psi \nvDash \theta$.
        \item $\theta$ is a \textbf{strong $(\varphi, \psi)$-contrast} if $\varphi \land \neg \psi \vDash \theta$ and $\neg \varphi \land \psi \vDash \neg \theta$.
        \item $\theta$ is a \textbf{$(\varphi, \psi)$-\common} if $\varphi \land \neg \psi \vDash \theta$ and $\neg \varphi \land \psi \vDash \theta$.
    \end{enumerate}
\end{definition}
Strong contrasts are properties that all models of $\varphi$ (or more precisely, $\varphi \land \neg \psi$), and none of the models of $\psi$, have. Weak contrasts, on the other hand, are properties that all models of, say, $\varphi$ have, but not all models of $\psi$ do. Likenesses are properties that models of both formulas have.

For global contrastive explanations, the output formulas correspond to contrasts and likenesses in a nice way. 
Let $(\theta, \theta', \chi)$ be an output of the global contrastive explanation problem (Definition \ref{def:globaltwoinputproblem}). We have $ \theta \wedge \chi \wedge \theta' \equiv \varphi \wedge \neg \psi \wedge \neg \varphi \wedge \psi \to \bot $, which means $\theta \wedge \neg \varphi \wedge \psi \to \bot$, namely $\neg \varphi \wedge \psi \to \neg \theta$, i.e. $\theta$ is always a strong $(\varphi, \psi)$-contrast. By symmetry $\theta'$ is a strong $(\psi, \varphi)$-contrast. It is also easy to see that $\chi$ is a $(\varphi, \psi)$-\common. 
The following result further shows that if we assume the output formulas are in CNF, then the individual clauses of $\theta$ and $\theta'$ are all weak contrasts whereas the clauses of $\chi$ are \common es of $\varphi$ and $\psi$.

\begin{theorem}\label{thm:weakcontrastsplit}
    Let $\varphi, \psi \in \mathrm{PL}[\tau]$ and let $(\theta, \theta', \chi)$ be the output of the global contrastive explanation problem with input $(\varphi, \psi)$. Further assume that $\theta$, $\theta'$ and $\chi$ are in \(\mathrm{CNF}\). Then \textbf{(1)} each clause of $\theta$ is a weak $(\varphi, \psi)$-contrast, \textbf{(2)} each clause of $\theta'$ is a weak $(\psi, \varphi)$-contrast, and \textbf{(3)} each clause of $\chi$ is a $(\varphi, \psi)$-likeness.
\end{theorem}
\begin{proof}
    We start with 1. Let $C \in \theta$ be a clause. Now $\varphi \land \neg \psi \vDash \theta \vDash C$ so $C$ is either a weak $(\varphi, \psi)$-contrast or a $(\varphi,\psi)$-\common. Assume for a contradiction that $C$ is a $(\varphi, \psi)$-\common. We have several cases.
    \begin{enumerate}
        \item [(a)] If $C \in \theta'$, then by moving $C$ from both $\theta$ and $\theta'$ to $\chi$ we preserve the equivalences of condition 1 of the contrastive explanation problem. Furthermore,
        \[
            \size(\theta \setminus \{C\}) + \size(\theta' \setminus \{C\}) + \size(\chi \land C) < \size(\theta) + \size(\theta') + \size(\chi),
        \]
        which is a contradiction by condition 2 of the problem.
        \item [(b)] If $C \in \chi$, then by removing $C$ from $\theta$ we get $\theta \setminus \{C\} \land \chi \equiv \theta \land \chi$ so we again contradict condition 2 of the problem.
        \item [(c)] If $C \notin \theta'$ and $C \notin \chi$, then by moving $C$ from $\theta$ to $\chi$ we again preserve the equivalences of condition 1 since $C$ is a $(\varphi, \psi)$-\common. Furthermore,
        \[
            \size(\theta \setminus \{C\}) + \size(\theta') + \size(\chi \land C) = \size(\theta) + \size(\theta') + \size(\chi),
        \]
        and $\size(\chi \land C) > \size(\chi)$. This is a contradiction with condition 3 of the contrastive explanation problem. Thus $C$ is a weak $(\varphi, \psi)$-contrast, proving claim 1.
    \end{enumerate}
    Claim 2 follows from the above by symmetry, while claim 3 follows directly from condition 1 of the problem. 
\end{proof}
\begin{remark}
    Minimizing $\size(\theta) + \size(\theta') + \size(\chi)$ is not the only conceivable minimality condition for the formulas $\theta \land \chi$ and $\theta' \land \chi$. If one thinks of these as two formulas to be minimized, then another natural condition could be $\size(\theta \land \chi) + \size(\theta' \land \chi)$. For the global contrastive explanation problem, changing to this alternate condition would forfeit the property of Theorem \ref{thm:weakcontrastsplit}, but gain a property, where clauses that are strong contrasts can be easily identified from the output.
\end{remark}
We also note that the same kind of property holds for the contrastive explanation problem. The difference is that here, the contrasts are found between $S$ and $S'$ rather than $\varphi$ and $\psi$. 

\begin{theorem}\label{thm:twoinputweakcontrastsplit}
    Let $\varphi, \psi \in \mathrm{PL[\tau]}$ and let $(\theta, \theta', \chi)$ be the output of the contrastive explanation problem with input $(S, S',\varphi, \psi)$. Further assume that $\theta$, $\theta'$ and $\chi$ are in \(\mathrm{CNF}\). Then \textbf{(1)} each clause of $\theta$ is a weak $(S, S')$-contrast, \textbf{(2)} each clause of $\theta'$ is a weak $(S', S)$-contrast, and \textbf{(3)} each clause of $\chi$ is a $(S, S')$-likeness.
\end{theorem}
\begin{proof}
    We first note that since $S \vDash \varphi \land \neg \psi$ and $S' \vDash \neg \varphi \land \psi$, we have $S \land \neg S' \equiv S$. Now for a clause $C \in \theta$, we have $S \vDash \theta \vDash C$ and we can use the same arguments as in the proof of Theorem \ref{thm:weakcontrastsplit} to prove the claim. 
\end{proof}
Note that Theorem \ref{thm:twoinputweakcontrastsplit} also works for the counterfactual problems if we consider an existentially quantified $S'$ to be implicitly present in the definitions.

\subsection{Link to Existing Contrastive Explanations}

In this subsection we link the counterfactual difference problem to existing notions in the literature, such as CXp. For example, we show that if the input $S$ defines an assignment and output formulas $\theta$, $\theta'$ and $\chi$ are given in $\mathrm{CNF}$, then the output gives the minimal changes required to flip the truth values of $\varphi$ and $\psi$. The difference from CXp is that our definition gives a cardinality-minimal solution rather than a subset-minimal one.

We start with the following theorem which shows that in the special case of the counterfactual difference problem where \(S\) defines a partial assignment, the optimal solutions also define partial assignments. Example \ref{example:sea-bird-example} shows that in general we cannot require \(\chi\) to even define a partial assignment.

\begin{theorem}\label{thm:partial-partial}
    Let $\varphi, \psi \in \mathrm{PL}[\tau]$ and let $S \subseteq \mathrm{PL}[\tau]$ define a partial assignment. Let $(\theta, \theta', \chi)$ be the output of the counterfactual difference problem $(S, \varphi, \psi)$. Further assume that $\theta$, $\theta'$ and $\chi$ are in $\mathrm{CNF}$. Then \textbf{(1)} the formulas $\theta$ and $\theta'$ are partial assignments and \textbf{(2)} the formulas $\theta \land \chi$ and $\theta' \land \chi$ define partial assignments.
The same also holds for the counterfactual contrastive explanation problem.
\end{theorem}
\begin{proof}
    We prove the claim for the counterfactual difference problem. The proof for the counterfactual contrastive explanation problem is very similar so we note the differences as they come up. 
    
    We first show that $\theta$ is a partial assignment. Assume that $\theta$ has a clause $C$ with $C \geq 2$. We show how to obtain a smaller partial assignment $\theta_\ast$ with $\theta_\ast \land \chi \equiv S$, thus contradicting condition 3 of the problem.
    
    First assume that $C$ has no literals from $S$. Now we have $S \vDash \theta \land \chi \vDash C$. All literals in $C$ are either in direct opposition with $S$ or use proposition symbols not mentioned in $S$. The only remaining option is that $C \equiv \top$ and thus $C$ can be safely removed from $\theta$.
    
    Now assume $C$ has at least two different literals $\ell_1$ and $\ell_2$ from $S$. 
    We show that removing $\ell_2$ from $C$ maintains the equivalence $\theta_\ast \land \chi \equiv S$. For $\theta_\ast \land \chi \vDash S$ note that $C \setminus \{\ell_2\} \vDash C$ so $\theta_\ast \vDash \theta$. For the direction $S \vDash \theta_\ast \land \chi$ we have $S \vDash \ell_1 \vDash C \setminus \{\ell_2\}$ since $\ell_1 \in C$ and for all other clauses $D \in \theta_\ast \land \chi$ we have $S \vDash \theta \land \chi \vDash D$.
    
    Finally, if $C$ has exactly one literal $\ell$ from $S$, we can safely replace $C$ with just the literal $\ell$. Now $\ell \vDash C$ so $\theta_\ast \land \chi \vDash \theta \land \chi \vDash S$. For the other direction, since $\ell \in S$ we have $S \vDash \ell$ and for the other clauses $D \in \theta_\ast \land \chi$ we have $S \vDash \theta \land \chi \vDash D$. By removing or replacing each $C$ as above, we end up with a smaller partial assignment $\theta_\ast$ as desired.

    We move on to $\theta'$. If $\theta'$ is not a partial assignment, we can find a partial assignment $\theta'_\ast$ with $\theta'_\ast \vDash \theta'$ as follows. Since $S$ is satisfiable, by condition 2 of the counterfactual difference problem, $\theta' \land \chi$ is also satisfiable. Let $s'$ be a satisfying assignment. For every clause $C$ in $\theta'$ choose a literal $\ell$ with $\ell \in s'$ and replace $C$ with $\ell$ in $\theta'$. This procedure produces a partial assignment $\theta_\ast \subseteq s'$. Now $\theta'_\ast \vDash \theta'$ so condition 1 of the problem holds. Since $s' \vDash \theta'_\ast \land \chi$, condition 2 also holds. Finally $\theta'_\ast$ has smaller size than $\theta'$ as at least one literal has been removed. This contradicts condition 3.

    For the counterfactual contrastive explanation problem, we handle $\theta'$ as above. We handle $\theta$ in the same way as $\theta'$, but instead of some satisfying assignment $s'$ we use the input partial assignment $S$. 

    By the definition of the counterfactual difference problem, $\theta \land \chi \equiv S$ is equivalent to a partial assignment. It remains to show that $\theta' \land \chi$ is also equivalent to some partial assignment. 

    Assume for  contradiction that $\theta' \land \chi$ is not equivalent to a partial assignment. We will show that we can find $\theta_\ast$, $\theta'_\ast$ and $\chi_\ast$ that satisfy conditions 1 and 2 of the problem but have smaller total size than $\theta$, $\theta'$ and $\chi$, thus contradicting condition 3.

    Since $\theta' \land \chi$ is satisfiable, any clause $C \in \chi$ has at least one literal $\ell$ with $\theta' \land (\chi \setminus \{C\}) \land \ell$ satisfiable. We call such literals \emph{relevant}. If each clause has only one relevant literal, then $\theta' \land \chi$ is equivalent to a partial assignment. We now fix a clause $C \in \chi$ with at least two relevant literals.

    Since $S \vDash \theta \land \chi$ (for both problems), the clause $C$ must have at least one literal from $S$, or be equivalent to $\top$ and thus easily removable. If $C$ has a relevant literal $\ell \in S$, then we replace $C$ with the literal $\ell$ in $\chi$ to obtain $\chi_\ast$. Now $\theta \land \chi_\ast \vDash \theta \land \chi \vDash S$ and $S \vDash \ell \land \theta \land \chi \vDash \theta \land \chi_\ast$ so $\theta \land \chi_\ast \equiv S$. For the counterfactual contrastive explanation problem only the direction $S \vDash \theta \land \chi_\ast$ holds and is needed here. On the other hand, $\chi_\ast \vDash \chi$ guarantees $\theta' \land \chi_\ast \vDash \theta' \land \chi \vDash \neg \varphi \land \psi$ so condition 1 holds. The relevance of $\ell$ guarantees that $\theta' \land \chi_\ast$ is satisfiable so condition 2 holds.

    Otherwise, let $\ell \in C$ be a relevant literal with $\ell \notin S$ and let $\ell' \in C \cap S$. We remove $C$ from $\chi$, add $\ell'$ to $\theta$ and add $\ell$ to $\theta'$ as singleton clauses. Now $\theta_\ast \land \chi_\ast \equiv S$ (or just $S \vDash \theta_\ast \land \chi_\ast$) as above. For $\theta'_\ast \land \chi_\ast$, it suffices to note that $\ell \vDash C$ so $\theta'_\ast \land \chi_\ast \vDash \theta' \land \chi \vDash \neg \varphi \land \psi$. For the size condition, note that $C$ has at least three literals: $\ell' \in S$, the relevant literal $\ell$ and at least one more relevant literal. Therefore the size of $\chi$ is reduced by at least 3 by removing $C$ while the sizes of $\theta$ and $\theta'$ both increase by 1 with the additions made. 

    In all cases above we have found $\theta_\ast$, $\theta'_\ast$ and $\chi_\ast$ that satisfy conditions 1 and 2 of the problem while having smaller total size for condition 3. This is a contradiction. Thus $\theta' \land \chi$ is equivalent to a partial assignment. 
\end{proof}

In the case where \(S\) defines a \(\tau\)-assignment, we get a stronger guarantee on optimal outputs for the counterfactual difference problem. 

\begin{theorem}\label{thm:assignment-assignment}
    Let $\varphi, \psi \in \mathrm{PL}[\tau]$, let $S \subseteq \mathrm{PL}[\tau]$ define a $\tau$-assignment and let $(\theta, \theta', \chi)$ be the output of the counterfactual difference problem with input $(S, \varphi, \psi)$. Further assume that $\theta$, $\theta'$ and $\chi$ are in $\mathrm{CNF}$. Now the formulas $\theta \land \chi$ and $\theta' \land \chi$ define $\tau$-assignments.
\end{theorem}
\begin{proof}
    By the definition of the problem, $\theta \land \chi$ defines the same $\tau$-assignment as $S$. From Theorem \ref{thm:partial-partial} we know that $\theta$ and $\theta'$ are partial assignments and $\theta' \land \chi$ defines a partial assignment. It remains to show that $\theta' \land \chi$ defines a full $\tau$-assignment. 

    Let $s'$ be the partial assignment defined by $\theta' \land \chi$ and assume for contradiction that there is $p \in \tau$ such that $p, \neg p \notin s'$. Since $S \equiv \theta \land \chi$ defines a full $\tau$-assignment, there is at least one occurrence of $p$ in $\theta \land \chi$. Furthermore, as all output formulas are in CNF, $p$ must occur as the literal $\ell$ given by the $\tau$-assignment defined by $S$.
    
    If $\ell$ occurs in the partial assignment $\theta$, then we can remove $\ell$ from $\theta$ and add it as a clause of $\chi$, obtaining new formulas $\theta_\ast$ and $\chi_\ast$. Clearly $\theta_\ast \land \chi_\ast \equiv S$. Additionally $\theta' \land \chi_\ast \vDash \theta' \land \chi \vDash \neg \varphi \land \psi$ so condition 1 of the problem holds for the new formulas. For condition 2, we note that $\theta' \land \chi$ defines a partial assignment not containing $p$ so $\theta' \land \chi_\ast \equiv \theta' \land \chi \land \ell$ is satisfiable. For condition 3, we note that the total size of the three formulas stays the same. Finally, for condition 4, we have $\size(\chi_\ast) > \size(\chi)$, which is a contradiction.

    If $\ell$ occurs in $\chi$, then it must be in a clause $C$ with at least two literals. Here we replace the clause $C$ with the literal $\ell$ in $\chi$, obtaining $\chi_\ast$ with smaller size. For condition 1, we see that $\ell \vDash C$ so $\theta \land \chi_\ast \vDash \theta \land \chi \vDash S$. On the other hand $S \vDash \ell$ and $S \vDash \theta \land \chi \setminus \{C\}$ so $S \vDash \theta \land \chi_\ast$. Finally $\theta' \land \chi_\ast \vDash \theta' \land \chi \vDash \neg \varphi \land \psi$ so condition 1 holds. For condition 2 we again need only note that $p$ does not occur in the partial assignment defined by $\theta' \land \chi$. The smaller size of $\chi_\ast$ contradicts condition 3 of the problem. 
\end{proof}

The above theorem does not hold for the counterfactual contrastive explanation problem. This is to be expected as the condition $S \vDash \theta \land \chi$ is not sufficient to force $\theta \land \chi$ to define an assignment, again highlighting the fact that the counterfactual contrastive explanation problem is concerned with comparing \emph{reasons} for the formulas $\varphi$ and $\psi$ rather than assignments that satisfy them.


To state the last result of this section, we define the notation $s \triangle \lambda = (s \setminus \lambda) \cup \{\overline{\ell} \mid \ell \in \lambda\}$, where $s$ and $\lambda$ are viewed as sets of literals with $\lambda \subseteq s$.
This result is of particular interest, as when the second input formula \(\psi\) is the negation \(\neg \varphi\) of the first, the set $\lambda$ of propositions to be flipped is a cardinality minimal CXp \cite{ignatiev20}.

\begin{theorem}\label{thm:assignment-correspondence}
    Let $s$ be a $\tau$-assignment and let $\varphi, \psi \in \mathrm{PL}[\tau]$ such that $\neg \varphi \land \psi$ is satisfiable. Assume that $s \vDash \varphi \land \neg \psi$ and let $(\theta, \theta', \chi)$ be the output of the counterfactual difference problem for \(\mathrm{PL}\) with input $(S_s, \varphi, \psi)$. Further assume that $\theta$, $\theta'$ and $\chi$ are in $\mathrm{CNF}$. Now $\theta' \land \chi$ defines a $\tau$-assignment $s'$ such that $s' = s \triangle \lambda$ for a cardinality-minimal set $\lambda$ of literals with $s \vDash \lambda$, $s \triangle \lambda \vDash \neg \varphi \land \psi$.
\end{theorem}
\begin{proof}
    We know from Theorem \ref{thm:assignment-assignment} that $\theta' \land \chi$ defines a $\tau$-assignment $s'$. We will show that given two candidates for $s'$, the one with fewer differences compared to $s$ is the preferred candidate due to the formula size condition 3 of the problem. This shows that the optimal output defines an assignment $s'$ with a minimum number of differences compared to $s$.
    
    Let $s_1, s_2$ be $\tau$-assignments with $s_i = s \triangle \lambda_i$. We assume $|\lambda_1| < |\lambda_2|$. Let $\theta_2, \theta'_2, \chi_2$ be the optimal output formulas corresponding to $s_2$. For every $\ell \in \lambda$ we have $s \vDash \ell$ and $s_2 \vDash \overline{\ell}$. Moreover, we have $\theta_2 \land \chi_2 \vDash s \vDash \ell$ and $\theta'_2 \land \chi_2 \vDash s_2 \vDash \overline{\ell}$. Since all output formulas are in \(\mathrm{CNF}\), inferring a literal requires a positive occurrence of that literal. In other words, for every $\ell \in \lambda_2$, both $\ell$ and $\overline{\ell}$ must occur somewhere in $\theta_2$, $\theta'_2$ or $\chi_2$. On the other hand, literals in $s \setminus \lambda$ must occur at least once. Thus the total size of $\theta_2$, $\theta'_2$ and $\chi_2$ must be at least $2|\lambda_2|+(|\tau|-|\lambda_2|) = |\tau| + |\lambda_2|$.

    For $s_1$, consider the formulas $\theta_1 = \bigwedge \lambda_1$, $\theta'_1 = \bigwedge \overline{\lambda_1}$ and $\chi_1 = \bigwedge (\tau \setminus \lambda_1)$. Clearly these formulas define the assignments $s$ and $s_1$ as desired. Furthermore, for the total size we have $\size(\theta_1) + \size(\theta'_1) + \size(\chi) = |\tau|+ |\lambda_1| < |\tau|+ |\lambda_2| \leq \size(\theta_2) + \size(\theta'_2) + \size(\chi)$. The claim follows. 
\end{proof}

\subsection{Computational Complexity}\label{subsec:complexity}

We proceed to study the computational complexity of our problems for propositional logic. As per usual, we study the complexity of the related decision problems: instead of finding a minimal solution to the given problem, we want to decide whether there is a solution of size at most \(k\), where \(k\) is part of the input.

We start with global contrastive explainability problem. Recall that in the propositional setting we require the output formulas to be CNF-formulas.

\begin{theorem}
    The global contrastive explanation problem is \(\Sigma_2^p\)-complete.
\end{theorem}
\begin{proof}
    We will give a simple reduction from the minimal formula size problem for CNF-formulas, which was proved to be \(\Sigma_2^p\)-complete in \cite{UMANS2001597}. Let \((\varphi,k)\) be an input to the latter problem. Consider the following input \((\bot,\varphi,k)\) to the global contrastive explainability problem. If \((\theta,\theta',\chi)\) is a witness for this input, then \(\theta \equiv \bot\) and \(\theta' \land \chi \equiv \varphi\). It is now easy to see that the output of \((\varphi,k)\) is \texttt{yes} iff the output of \((\bot,\varphi,k)\) is \texttt{yes}. 
\end{proof}

For the next hardness result we use a result from \cite{explainability1} which states that the so-called \textbf{local explainability problem} for \(\mathrm{PL}\) is \(\Sigma_2^p\)-complete. In this problem the input is a tuple \((s,\varphi,k)\), where \(s\) is an assignment, such that \(s \models \varphi\) and the goal is to determine whether there exist a formula \(\psi\) of \(\mathrm{PL}\) such that \(\mathrm{size}(\psi) \leq k\) and \(s \models \psi \models \varphi\). 
The following result demonstrates that this problem is a special case of the contrastive explainability problem and the separability problem.

\begin{theorem}
    Both the contrastive explanation problem and the minimal separator problem for \(\mathrm{PL}\) are \(\Sigma_2^p\)-complete.
\end{theorem}
\begin{proof}
    Upper bound is clear. For the lower bounds, consider an instance \((s,\varphi,k)\) of the local explainability problem. Let \(\varphi_s\) be a conjunction of literals which defines \(s\). For contrastive explainability problem the hardness follows from the observation that \((\varphi_s,\bot,\varphi,\bot,k)\) is an instance of the contrastive explainability problem for which the output is \texttt{yes} iff the output of \((s,\varphi,k,1)\) is \texttt{yes}. In the right-to-left direction we use a result from \cite{explainability1} which shows that if there exist a formula \(\psi\) with \(\size(\psi) \leq k\) and \(s \models \psi \models \varphi\), then there is also one which is a conjunction of literals. For separability problem we get the hardness by considering instances of the form \((\varphi_s,\neg \varphi,k)\). 
\end{proof}

We move on to consider counterfactual contrastive and difference problems. Theorem \ref{thm:partial-partial} shows that if the input $S$ defines a partial assignment, then the formulas $\theta \land \chi$ and $\theta' \land \chi$ from the output $(\theta, \theta', \chi)$ also define partial assignments. Thus it is natural to define the \textbf{simplified counterfactual contrastive explanation problem} and the \textbf{simplified counterfactual difference problem} by modifying  Definitions \ref{def:problem:one_input_consequence} and \ref{def:problem:one_input_equivalence} as follows. First, we require that $S$ defines a partial assignment. Secondly, the output formulas $\theta$, $\theta'$ and $\chi$ are required to be partial assignments. 

Our next result shows that already the simplified problems are $\Sigma^p_2$ complete. We leave the complexity of the non-simplified problems for future work. 

\begin{theorem}\label{thm:simplified-complexity}
The simplified counterfactual contrastive explanation problem and the simplified counterfactual difference problem are \(\Sigma_2^p\)-complete.
\end{theorem}
\begin{proof}
    Upper bounds are clear. For the lower bounds we will adapt the proof of Theorem 6 in \cite{explainability1}. That is, we will give a reductions from \(\Sigma_2 \mathrm{SAT}\), which is well-known to be \(\Sigma_2^p\)-complete.  Consider an instance
    \[\exists p_1 \dots \exists p_n \forall q_1 \dots \forall q_m \xi(p_1,\dots,p_n,q_1,\dots,q_m)\]
    of \(\Sigma_2 \mathrm{SAT}\). We start by introducing for every existentially quantified Boolean variable \(p_i\) a new proposition symbol \(p^d_i\) which intuitively speaking represents the negation of \(p_i\). Denoting \(\xi(p_1,\dots,p_n,q_1,\dots,q_m)\) simply by \(\xi\), we define 
    \[\varphi\ :=\ \bigwedge_{i=1}^n (p_i \lor p^d_i) \land \xi\]
    Let
    \[s := \{\neg p_i \mid 1 \leq i \leq n\} \cup \{\neg p^d_i \mid 1 \leq i \leq n\}.\]
    Recall the notation $s \triangle \lambda = (s \setminus \lambda) \cup \{\overline{\ell} \mid \ell \in \lambda\}$, where $s$ is a set of literals and $\lambda \subseteq s$. As part of the proof of Theorem 6 in \cite{explainability1} it was essentially established that there exists a set \(\lambda\) of literals of size \(n\) such that \(s \Delta \lambda \models \varphi\) iff the original instance of \(\Sigma_2 \mathrm{SAT}\) is true.
    
    Let 
    \[\psi := \bigwedge_{i = 1}^n (\neg p_i \land \neg p^d_i).\]
    Note that \(s \equiv \psi \models \neg \varphi\). Consider now the tuple \((s,\psi, \neg \varphi, 3n)\). We will show that for both problems the output of this instance is \texttt{yes} iff there exists a set \(\lambda\) of literals of size \(n\) such that \(s \Delta \lambda \models \varphi\).
    
    Suppose first that there exists a set \(\lambda\) of literals of size \(n\) such that \(s \Delta \lambda \models \varphi\). Pick \(\theta, \theta', \chi\) in such a way that
    \[\theta \land \chi := \psi\]
    and
    \[\theta' \land \chi := \bigwedge_{p \in \lambda} p \land \chi.\]
    Now
    \[\mathrm{size}(\theta) + \mathrm{size}(\theta') + \mathrm{size}(\chi) = 3n\]
    and \(\theta \land \chi \equiv s\).
    
    Suppose then that there exists a pair \((\theta,\theta')\) and \(\chi\) such that the following conditions hold.
    \begin{itemize}
        \item \(s \models \theta \land \chi \models \psi \land \neg \varphi\) and \(\theta' \land \chi \models \neg \psi \land \varphi\).
        \item \(\mathrm{size}(\theta) + \mathrm{size}(\theta') + \mathrm{size}(\chi) \leq 3n\).
        \item \(\theta,\theta'\) and \(\chi\) are conjunctions of literals.
    \end{itemize}
    First observation that we make is that since \(s \models \theta \land \chi \models \psi\), we must have that \(\theta \land \chi \equiv \psi\) (i.e., \(\theta \land \chi \equiv s\)), which means that \(\mathrm{size}(\theta) + \mathrm{size}(\chi) = 2n\). Thus \(\mathrm{size}(\theta') \leq n\). On the other hand, since \(\theta' \land \chi \models \varphi\), \(\theta'\) must contain for each \(1 \leq i \leq n\) either \(p_i\) or \(p^d_i\), because any assignment which makes \(\varphi\) true must map either \(p_i\) or \(p^d_i\) to \(1\). Since \(\mathrm{size}(\theta') \leq n\) , \(\theta'\) must contain for each \(1 \leq i \leq n\) either \(p_i\) or \(p^d_i\) and no other proposition symbols. Letting \(\lambda\) denote the corresponding set of literals of size \(n\), we get that \(s \Delta \lambda \models \varphi\). 
\end{proof}

\section{Implementation and Case Studies}

In this section we outline our Answer Set Programming (ASP) based implementation for selected explanation problems, followed by case studies where we use it to demonstrate how our problems work on real-world instances.

\paragraph{\bf{Implementation}}
Due to the computational complexity of the considered problems, obtaining a polynomial time dedicated algorithm seems unlikely. However, the inherent complexity closely matches the one of ASP~\cite{EiterIK09} and it is thus natural to implement our notions in this formalism to obtain a prototypical implementation.

We thus implemented Definitions \ref{def:globaltwoinputproblem}, \ref{def:problem:one_input_consequence}, and \ref{def:problem:one_input_equivalence} in ASP. The encodings as well as a Python script for reproducing the case studies can be found online at \url{https://github.com/tlyphed/general_contrastive_exp}.
Due to length constraints we cannot cover the ASP encoding in detail but we will give a brief overview.

The idea behind the encodings is to guess the output formulas $\theta, \theta',\chi$ as CNF and check that the required entailments and equivalences hold. Those checks are done using the \emph{saturation technique}~\cite{EiterIK09}, which is an encoding method for expressing $\coNP$ problems in ASP.
The criteria regarding the sizes of the formulas are then formulated using weak constraints, which essentially give preference to formulas which adhere to those criteria.

\paragraph{\bf{Case Studies}}
We used our implementation to compute contrastive explanations for decision trees that were obtained from three real-world classification datasets. The datasets that we used were Iris \cite{iris}, Wine \cite{wine} and Glass \cite{glass}. We used the \texttt{scikit-learn} Python library \cite{scikit-learn} for training the classifiers. For all datasets we used \(80\%\) of the data for training and \(20\%\) for testing. To keep the learning task simple, we fine-tuned only the depth of the decision trees, selecting the smallest depth that achieved the highest accuracy on the test data.

Having trained the decision trees, we proceeded as follows. For each dataset we picked two classes \(c,c'\) at random and then used the learned decision tree \(\mathcal{T}\) to form class formulas for \(c,c'\). The propositional symbols in these class formulas correspond to pivots used by \(\mathcal{T}\). We also picked a random instance from the test data that was classified as \(c\) by \(\mathcal{T}\). We then ran our implementations for global contrastive explanation problem (GCE), counterfactual contrastive explanation problem (CCE) and counterfactual difference problem (CD) using the class formulas and the instance as inputs. For CCE and CD, we additionally constrained the output formulas to be partial assignments. This choice is motivated by Theorems \ref{thm:partial-partial} and \ref{thm:assignment-assignment}, since in our case studies \(S\) (the instance) corresponds to a conjunction of literals. The results are summarized in Table \ref{tab:simple-table}.

\begin{table}[!ht]
\centering
\caption{Summary of our case studies. Note that we have Booleanized the instances using the pivots learned by the corresponding decision tree.}
\scalebox{0.85}{ 
\begin{tabular}{cccc}
\toprule
\textbf{Case} & \textbf{GCE} & \textbf{CCE} & \textbf{CD} \\
\midrule
\makecell[l]{dataset: Iris \\ depth: \(4\) \\ instance: \\ \(\{\neg p_1, p_2, p_3, p_4\}\)} & \makecell[l]{\(\theta : (\neg p_2 \lor p_3) \land (p_2 \lor p_4)\) \\ \(\theta' : (\neg p_2 \lor \neg p_3) \land (p_2 \lor \neg p_4)\) \\ \(\chi : \neg p_1\)} & \makecell[l]{\(\theta : p_3\) \\ \(\theta' : \neg p_3\) \\ \(\chi : \neg p_1 \land p_2\)} & \makecell[l]{\(\theta : p_3\) \\ \(\theta' : \neg p_3\) \\ \(\chi : \neg p_1 \land p_2 \land p_4\)} \\
\hline
\makecell[l]{dataset: Wine \\ depth: \(3\) \\ instance: \\ \(\{p_1 , p_2 , p_3 , \neg p_4, p_5\}\)} & \makecell[l]{\(\theta : (p_1 \lor \neg p_4) \lor (p_1 \lor p_5)\) \\ \(\theta' : \neg p_1 \land p_4\) \\ \(\chi : (\neg p_1 \lor p_3) \land (\neg p_1 \lor p_2)\)} & \makecell[l]{\(\theta : p_1\) \\ \(\theta' : \neg p_1 \land p_4\) \\ \(\chi : p_2 \land p_3\)} & \makecell[l]{\(\theta : p_1 \land \neg p_4\) \\ \(\theta' : \neg p_1 \land p_4\) \\ \(\chi : p_2 \land p_3 \land p_5\)} \\
\hline
\makecell[l]{dataset: Glass \\ depth: \(3\) \\ instance: \\ 
\(\{p_1 , p_2 , \neg p_3 , \neg p_4, \) \\ \(\neg p_5 , p_6 , p_7\}\)} 
& \makecell[l]{\(\theta : (\neg p_1 \lor p_2)  \land  (p_1 \lor \neg p_5)\) \\ \(\land  (\neg p_1 \lor \neg p_3)  \land  (p_1 \lor \neg p_7)\) \\ \(\theta' : (p_1 \lor p_5)  \land  (p_1 \lor \neg p_6)\) \\ \(\land  (\neg p_1 \lor p_4)  \land (\neg p_1 \lor \neg p_2)\) \\ \(\chi : \top\)} & \makecell[l]{\(\theta : p_2\) \\ \(\theta' : \neg p_2 \land p_4\) \\ \(\chi : p_1 \land \neg p_3\)} & \makecell[l]{\(\theta : p_2 \land \neg p_4\) \\ \(\theta' : \neg p_2 \land p_4\) \\ \(\chi : p_1 \land \neg p_3 \land \neg p_5\) \\ \(\land p_6 \land p_7\)} \\
\bottomrule
\end{tabular}
}
\smallskip
\label{tab:simple-table}
\end{table}
We make some general remarks about Table \ref{tab:simple-table}. For GCE we see both simple and more complex likenesses. For example, in the case of Glass, the optimal output had no common clauses for the two class formulas. In each case we see that CCE and CD have chosen the same propositions to flip, although we know by Example \ref{example:sea-bird-example} that this is not necessary. The CCE outputs are more informative here since they also leave out unnecessary propositions from the explanations.

We examine the formulas in more detail for the Iris dataset. Each instance is an iris flower and the goal is to classify them into one of three species: setosa, versicolor, or virginica. We formed class formulas for versicolor and virginica. From GCE we see e.g. that the class formulas have \(\neg p_1\) in common, where 
\[p_1 := \text{``petal length is \(\leq 2.45\)cm''}.\]
So the class formulas for versicolor and virginica both agree that petal length should be more than \(2.45\)cm. From CCE and CD, we see that the instance \(\{\neg p_1 , p_2 , p_3 , p_4\}\) was classified as versicolor, because \(p_3 \land \neg p_1 \land p_2\), where 
\[p_2 := \text{``petal length is \(\leq 4.75\)cm''}\]
and 
\[p_3 := \text{``petal width is \(\leq 1.65\)cm''}.\]
So the explanation is that the flower has petal length between \(2.45\)cm and \(4.75\)cm and petal width at most \(1.65\)cm. If the petal length had been more than \(4.75\)cm, the flower would have been classified as virginica.

\section{Conclusion}

In this work, we introduced a logic-based framework for contrastive explanations that formalizes various questions of the form ``Why \(P\) but not \(Q\)?''. Our framework encompasses both local and global contrastive explanations. We examined the theoretical properties of our definitions in detail in the important special case of propositional logic. Among other results, we showed that our framework subsumes a cardinality-minimal version of existing contrastive explanation approaches from \cite{darwiche23,ignatiev20}. In addition, we analyzed the computational complexity of the proposed problems and proved that, in the propositional setting, most of them are \(\Sigma_2^p\)-complete. Finally, we implemented our approach using Answer Set Programming
and demonstrated that our problems produce useful explanations on real-world
instances.

\medskip

\noindent \textbf{Acknowledgments.} Antti Kuusisto and Miikka Vilander received funding from the Research Council of Finland projects \emph{Explaining AI via Logic} (XAILOG) 345612 and \emph{Theory of computational logics} 324435, 328987, 352419, 352420. 
Tobias Geibinger is a recipient of a DOC Fellowship of the Austrian Academy of Sciences at the Institute of Logic and Computation at the TU Wien.
This work has benefited from the European Union's Horizon 2020 research and innovation programme under grant agreement No 101034440 (LogiCS@TUWien).\raisebox{-0.15cm}{\includegraphics[width=0.7cm]{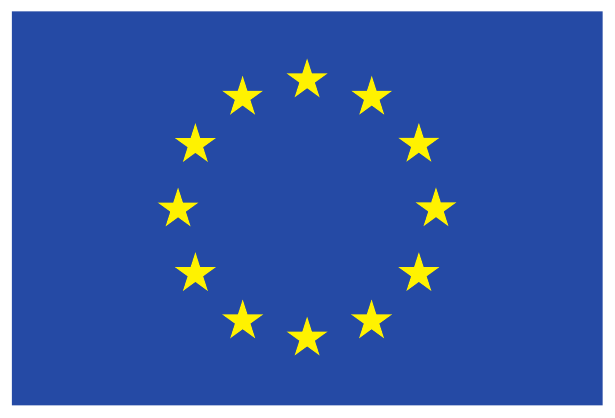}} In addition, this research was funded in part by the Austrian Science Fund (FWF) [10.55776/COE12].
The authors acknowledge TU Wien Bibliothek for financial support through its Open Access Funding Programme. 
The author names of this article are ordered based on alphabetical order.

\bibliographystyle{plainurl}
\bibliography{arxiv}

\end{document}